\documentclass{article}

\usepackage{microtype}
\usepackage{graphicx}
\usepackage{subfigure}
\usepackage{booktabs} 

\usepackage{hyperref}

\usepackage[accepted]{icml_style/icml2021}

\usepackage{latexsym,amsfonts,amsmath,amssymb,epsfig,tabularx,amsthm,dsfont,mathrsfs}

\usepackage{enumerate}
\usepackage[normalem]{ulem}

\theoremstyle{plain}

\newtheorem{thm}{Theorem}[section]
\newtheorem{lem}[thm]{Lemma}

\newtheorem{prop}[thm]{Proposition}
\newtheorem{cor}[thm]{Corollary}
\newtheorem{assumption}[thm]{Assumption}

\theoremstyle{definition}
\newtheorem{rem}[thm]{Remark}
\newtheorem{example}[thm]{Example}

\renewcommand{\P}{{\mathbb P}}

\newcommand{\mc}{\mathcal}

\renewcommand{\d}{\mathrm d}
\renewcommand{\rho}{\varrho}

\icmltitlerunning{Geometric Convergence of Elliptical Slice Sampling}

\begin{document}

\twocolumn[
\icmltitle{
Geometric Convergence of Elliptical Slice Sampling
}




\begin{icmlauthorlist}
	\icmlauthor{Viacheslav Natarovskii}{to}
	\icmlauthor{Daniel Rudolf}{to}
	\icmlauthor{Bj\"orn Sprungk}{goo}
\end{icmlauthorlist}

\icmlaffiliation{to}{Institute for Mathematical Stochastics, Georg-August-Universit\"at G\"ottingen, G\"ottingen, Germany}
\icmlaffiliation{goo}{Faculty of Mathematics and Computer Science, Technische Universit\"at Bergakademie Freiberg, Germany}

\icmlcorrespondingauthor{Viacheslav Natarovskii}{vnataro@uni-goettingen.de}
\icmlcorrespondingauthor{Daniel Rudolf}{daniel.rudolf@uni-goettingen.de}
\icmlcorrespondingauthor{Bj\"orn Sprungk}{bjoern.sprungk@math.tu-freiberg.de}


\vskip 0.3in
]



\printAffiliationsAndNotice{}  

\begin{abstract}
For Bayesian learning
, given likelihood function and Gaussian prior, the elliptical slice sampler, introduced by Murray, Adams and MacKay 2010, provides a tool for the construction of a Markov chain for approximate sampling of the underlying posterior distribution. 
Besides of its wide applicability and simplicity its main feature is that no tuning is required. 
Under weak regularity assumptions on the posterior density we show 
that the corresponding Markov chain is geometrically ergodic and therefore yield qualitative convergence guarantees.
We illustrate our result for Gaussian posteriors as they appear in Gaussian process regression
, as well as in a setting of a multi-modal distribution. 
Remarkably, our numerical experiments indicate a dimension-independent performance of  elliptical slice sampling even in situations where our ergodicity result does not apply.
\end{abstract}

\section{Introduction}
Probabilistic modeling provides a versatile tool in the analysis of data and allows for statistical inference. In particular, in Bayesian approaches one is able to quantify model and prediction uncertainty by extracting knowledge from the posterior distribution through sampling. The generation of exact samples w.r.t. the posterior distribution is usually quite difficult, since it is in most scenarios only known up to a normalizing constant. Let  $\rho\colon \mathbb{R}^d \to (0,\infty)$ be determined by a likelihood function given some data (which we omit in the following for simplicity) as mapping from the parameter space into the non-negative reals and let $\mu_0=\mathcal{N}(0,C)$ be a Gaussian prior distribution on $\mathbb{R}^d$ with non-degenerate covariance matrix $C$, such that the posterior distribution $\mu$ on $\mathbb{R}^d$ takes the form
\begin{align}\label{equ:densities}
	\mu({\rm d} x) = \frac{\rho(x)}{Z} 
	\mu_0({\rm d}x),
\quad
Z := \int_{\mathbb{R}^d} \rho(x)\ \mu_0(\d x).
\end{align}
For convenience, we abbreviate the 
	former relation between the measures $\mu(\d x)$ and $\rho(x)\mu_0(\d x)$ 
	as
	$\mu(\d x) \propto \rho(x)\mu_0(\d x)$.
	
A standard approach for generating approximate samples w.r.t. $\mu$ is given by Markov chain Monte Carlo. The idea is to construct a Markov chain, which has $\mu$ as its stationary and limit distribution\footnote{Limit distribution in the sense that for $n\to \infty$ the distribution of the $n$th random variable of the Markov chain converges to $\mu$.}.
 For this purpose in machine learning (and computational statistics in general) 
 Metropolis-Hastings algorithms and slice sampling algorithms (which include Gibbs sampling) are classical tools, see, e.g., \cite{neal1993probabilistic,andrieu2003introduction,neal2003slice}. 

Murray, Adams and MacKay in  \cite{murray2010elliptical} introduced the elliptical slice sampler. On the one hand it is based on a Metropolis-Hastings method suggested by Neal \cite{neal1999regression} (nowadays also known as preconditioned Crank-Nicolson Metropolis \cite{cotter2013mcmc,rudolf2018generalization}) and on the other hand it is a modification of slice sampling with stepping-out and shrinkage \cite{neal2003slice}. Elliptical slice sampling is illustrated in \cite{murray2010elliptical} on a number of applications, such as Gaussian regression, Gaussian process classification and a Log Gaussian Cox process. Apart from its simplicity and wide applicability the main advantage of the suggested algorithm is that it performs well in practice and no tuning is necessary. In addition to that in many scenarios it appears as a building block and/or influenced methodological development of sampling approaches \cite{fagan2016elliptical,hahn2019efficient,bierkens2020boomerang,murray2016pseudo,nishihara2014parallel}.

However, despite the arguments for being reversible w.r.t. the desired posterior in \cite{murray2010elliptical} there is, to our knowledge, no theory guaranteeing indeed convergence of the corresponding Markov chain.
%
%
%
%
%
Under a tail and a weak boundedness assumption on $\rho$ we derive a small set and Lyapunov function which imply  geometric ergodicity by standard theorems for Markov chains on general state spaces, see e.g. chapter~15 in \cite{meyn2009markov} and/or \cite{hairer2011yet}.

Before we state our ergodicity result in Section~\ref{sec:conv_ess} we provide the algorithm and introduce notation as well as basic facts. Afterwards we state the detailed analysis, in particular, the strategy of proof as well as verify the two crucial conditions of having a Lyapunov function and a sufficiently large small set.
In Section~\ref{sec:Illus} 
we illustrate the applicability of our theoretical result in a fully Gaussian and multi-modal scenario. Additionally, we compare elliptical with simple slice sampling and different Metropolis-Hastings algorithms numerically. 
The experiments indicate dimension-independent statistical efficiency of elliptical slice sampling which will be the content of future research.


\section{Convergence of Elliptical Slice Sampling}
\label{sec:conv_ess}

We start with stating the transition mechanism/kernel of elliptical slice sampling in algorithmic form and provide our notation. Let $(\Omega,\mathcal{F},\mathbb{P})$ be the underlying probability space of all subsequently used random variables. For $a,b\in \mathbb{R}$ with $a<b$ let $\mathcal{U}[a,b]$ be the uniform distribution on $[a,b]$ and let $\mathcal{B}(\mathbb{R}^d)$ be the Borel $\sigma$-algebra of $\mathbb{R}^d$. Furthermore, the Euclidean ball with radius $R>0$ around $x\in\mathbb{R}^d$ is denoted by $B_R(x)$ and the Euclidean norm is given by $\Vert \cdot\Vert$.
\subsection{Transition Mechanism}
We use the function $p: \mathbb{R}^d \times \mathbb{R}^d \times [0, 2\pi] \to \mathbb{R}^d$ defined as
\begin{align}
\label{eq:p-def}
p(x,w,\theta) := \cos(\theta) \ x + \sin(\theta)\ w,
\end{align}
where, for fixed $x,w \in \mathbb{R}^d$, the map $\theta \mapsto p(x,w,\theta)$ describes an ellipse in $\mathbb{R}^d$ 
with conjugate diameters $x,w$. Furthermore, for $t\geq0$ let
\[
G_t
:=
\{x\in \mathbb{R}^d\colon \rho(x) \geq t\},
\]
be the (super-)level set of $\rho$ w.r.t. $t$. Using this notation a single transition of elliptical slice sampling from $x\in\mathbb{R}^d$ to $y$ is presented in 
Algorithm~\ref{alg:ESS}. Here $y\in\mathbb{R}^d$ is considered as a realization of a random variable $Y_x$.
\begin{algorithm}
	\caption{Elliptical Slice Sampler}
	\label{alg:ESS}
	\begin{algorithmic}[1]
		\INPUT current state $x \in \mathbb{R}^d$
		\OUTPUT next state $y$ as realization of a random variable $Y_x$ 
		\STATE draw $W \sim \mu_0$, call the result $w$;
		\STATE draw $T_x \sim \mc U[0,\rho(x)]$, call the result $t$;
		\STATE draw $\Theta \sim \mathcal{U}[0, 2\pi]$, call the result $\theta$;
		\STATE $\theta_{\min} \leftarrow \theta - 2\pi$
		\STATE $\theta_{\max} \leftarrow \theta$
		\WHILE{$p(x, w, \theta) \not\in G_t$}
		\IF{$\theta < 0$}
		\STATE $\theta_{\min} \leftarrow \theta$
		\ELSE
		\STATE $\theta_{\max} \leftarrow \theta$
		\ENDIF
		\STATE draw $\Theta \sim \mathcal{U}[\theta_{\min}, \theta_{\max}]$, set the result to $\theta$;
		\ENDWHILE
		\STATE $y \leftarrow p(x, w, \theta)$
	\end{algorithmic}
\end{algorithm}
Let us denote the transition kernel which corresponds to elliptical slice sampling by $E\colon \mathbb{R}^d \times \mathcal{B}(\mathbb{R}^d) \to [0, 1]$ and for $A\in\mathcal{B}(\mathbb{R}^d)$ observe that 
\[
E(x,A) = \frac{1}{\rho(x)}\int_0^{\rho(x)} \int_{\mathbb{R}^d} E_{x,w,t}(A)\ \mu_0({\rm d}w){\rm d}t,
\]
where 
\begin{align}
\label{eq:sub_kernel}
E_{x,w,t}(A) := \mathbb{P}(Y_x\in A\mid W=w,T_x=t)
\end{align}
is determined by steps 3-14 of Algorithm~\ref{alg:ESS}.
These steps of the algorithm determine the sampling mechanism on $G_t$ intersected with the ellipse by using a suitable adaptation of the shrinkage procedure, see \cite{neal2003slice,murray2010elliptical}.
Let $(X_n)_{n\in\mathbb{N}}$ be a Markov chain generated by Algorithm~\ref{alg:ESS}, that is, a Markov chain on $\mathbb{R}^d$ with transition kernel $E$.
Then, for any $n\in\mathbb{N}$, $A\in\mathcal{B}(\mathbb{R}^d)$ and  $x\in\mathbb{R}^d$ we have
\begin{equation}
\label{eq:n_step_trans_kernl}
\mathbb{P}(X_{n+1}\in A\mid X_1=x) = E^n(x,A),
\end{equation}
where $E^n$ is iteratively defined as
\begin{align}
\label{al:nth_iterate}
E^{n+1}(x,A) 
& 
= \int_{\mathbb{R}^d} E^n(z,A)
E(x,{\rm d}z),
\end{align}
with $E^0(x,A)  = \mathds{1}_A(x)$ denoting the indicator function of the set $A$.
\subsection{Main Result}
Before we formulate the theorem, we state the assumptions which eventually imply the convergence result.
\begin{assumption}
	\label{assum:reg}
	The function $\rho\colon\mathbb{R}^d\to (0,\infty)$ satisfies the following properties:
	\begin{enumerate}
		\item \label{assum:pi-bounded-away}
		It is bounded away from $0$ and $\infty$ on any compact set.
		\item	\label{assum:good-tails}
		There exists an $\alpha>0$ and $R>0$, such that
		\[
		B_{\alpha\|x\|}(0) \subseteq G_{\rho(x)}  \qquad \text{for } \|x\| >R. 
		\]
	\end{enumerate}
\end{assumption}
The boundedness condition from below and above of $\rho$ on compact sets is relatively weak and appears frequently in qualitative proofs for geometric ergodicity of Markov chain algorithms
, see e.g. \cite{roberts1996geometric}.
The second condition 
tells us that $\rho$ has a sufficiently nice tail behavior. It is satisfied if the tails are rotational invariant and monotone decreasing, e.g., like $\exp(-\kappa\Vert x \Vert)$ for arbitrary $\kappa>0$. For examples of $\rho$ which satisfy Assumption~\ref{assum:reg} we refer to Section~\ref{sec:Illus}.

For stating the geometric ergodicity of elliptical slice sampling we introduce the total variation distance of two probability measures $\pi,\nu$ on $\mathbb{R}^d$ as
\[
\Vert \pi-\nu \Vert_{\rm tv} := \sup_{\Vert f \Vert_{\infty}\leq 1} \left \vert \int_{\mathbb{R}^d} f(x) (\pi({\rm d} x)-\nu({\rm d}x)) \right \vert,
\]
where $\Vert f \Vert_{\infty} := \sup_{x \in \mathbb R^d} |f(x)|$ for $f\colon \mathbb R^d \to \mathbb R$.
\begin{thm}\label{thm:main-result}
	For elliptical slice sampling under Assumption~\ref{assum:reg} there exist constants $C > 0$ and $\gamma \in (0, 1)$, such that
	\begin{equation} \label{eq: main_est}
	\Vert E^n(x,\cdot)-\mu \Vert_{\rm tv} \leq C (1+\Vert x \Vert) \gamma^n, \quad \forall n\in\mathbb{N}, \forall x \in \mathbb{R}^d.
	\end{equation}
\end{thm}
\begin{rem}
	A transition kernel which satisfies an inequality as in \eqref{eq: main_est} is called geometrically ergodic, since the distribution of $X_{n+1}$, given that the initial state $X_1=x$, converges exponentially/geometrically fast to $\mu$. Here, the right-hand side depends on $x$ only via the term $1+\Vert x\Vert$. 
	We view this result as a qualitative statement telling us about exponential convergence of the Markov chain whereas we do not care too much about the constants $C>0$ and $\gamma\in(0,1)$. 
	The main reason behind this is, that the employed technique of proof does usually not provide sharp bounds on $\gamma$ and $C$, particularly regarding their dependence on the dimension $d$.
\end{rem}

\section{Detailed Analysis}
For proving geometric ergodicity for Markov chains on general state spaces we employ a standard strategy, which consists of the verification of a suitable small set as well as a drift or Lyapunov condition, see e.g. chapter~15 in \cite{meyn2009markov} or \cite{hairer2011yet}.
More precisely we use a consequence of the Harris ergodic theorem as formulated in \cite{hairer2011yet}, which provides a relatively concise introduction and proof of a geometric ergodicity result for Markov chains.

\subsection{Strategy of Proof}
To formulate the convergence theorem we need the notion of a Lyapunov function and a small set. For this let $P\colon \mathbb{R}^d \times \mc B(\mathbb{R}^d)\to[0,1]$ be a generic transition kernel.

We call a function $V\colon \mathbb{R}^d \to [0,\infty)$ \emph{Lyapunov function} of $P$ with $\delta \in [0,1)$ and $L\in [0,\infty)$ if for all $x\in\mathbb{R}^d$ holds
\begin{align}
\label{eq:lyapunov-definition}
PV(x)
:=
\int_{\mathbb{R}^d} V(y) \ P(x, \d y)
\leq 
\delta V(x) + L.
\end{align}

Furthermore, a set $S\in\mathcal{B}(\mathbb{R}^d)$ is a \emph{small set w.r.t. $P$} and a non-zero finite measure $\nu$ on $\mathbb{R}^d$, if
\[
P(x,A) \geq \nu(A), \qquad
\forall x \in S, A \in \mathcal{B}(\mathbb{R}^d).
\]
%
%
%
%
%

With this terminology we can state a consequence of Theorem~1.2 in \cite{hairer2011yet}, which we justify for the convenience of the reader in Section~\ref{sec:suppl-thm} of the supplementary material.

\begin{prop}
	\label{prop:Harris_erg_thm}
	Suppose that for a transition kernel $P$ there is a Lyapunov function $V\colon \mathbb{R}^d \to [0,\infty)$ with $\delta\in [0,1)$ and $L\in[0,\infty)$ (\eqref{eq:lyapunov-definition} is satisfied). Additionally, for some constant $R>2L/(1-\delta)$  let
	\begin{equation}\label{eq:S_R}
	S_R := \{ x\in\mathbb{R}^d \colon V(x) \leq R \}
	\end{equation}
	be a small set w.r.t. $P$ and a non-zero measure $\nu$ on $\mathbb{R}^d$. Then, there is a unique stationary distribution $\mu_{\star}$ on $\mathbb{R}^d$, that is, for all $A\in\mathcal{B}(\mathbb{R}^d)$
	\[
	\mu_{\star} P (A) := \int_{\mathbb{R}^d} P(x,A)\mu_{\star}({\rm d}x) = \mu_{\star}(A),
	\]
	and there exist constants $\gamma\in(0,1)$ as well as $C<\infty$ such that
	\[
	\Vert P^n(x,\cdot) - \mu_{\star} \Vert_{\rm tv}
	\leq C (1 + V(x)) \gamma^n, \quad \forall n\in\mathbb{N}, \forall x \in \mathbb{R}^d.
	\]
	(Here $P^n$ is the $n$-step transition kernel defined as in \eqref{al:nth_iterate}.)
\end{prop}
From the arguments of reversibility of elliptical slice sampling w.r.t. $\mu$ derived in \cite{murray2010elliptical} we know already that $\mu$ is a stationary distribution w.r.t. the transition kernel $E$. 
The idea is now to first detect a suitable Lyapunov function $V$ of $E$ satisfying \eqref{eq:lyapunov-definition} for $P=E$ and a $\delta\in[0,1)$ and $L\in[0,\infty)$ and, having this, proving that the corresponding set $S_R$ from \eqref{eq:S_R} is a small set w.r.t. $E$ and a suitable measure $\nu$.

\subsection{Lyapunov Function}
Besides the usefulness of a Lyapunov function in the context of geometric convergence of Markov chains as in Proposition~\ref{prop:Harris_erg_thm} it arises to derive certain stability properties, e.g., it crucially appears in the perturbation theory of Markov chains in measuring the difference of transition kernels \cite{rudolf2018perturbation,medina2020perturbation}.

We start with the following abstract proposition inspired by Lemma 3.2 in \cite{hairer2014spectral}, see also Proposition~3 in \cite{hosseini2018spectral}.

\begin{prop}\label{prop:Lyapunov}
	Let $P$ be a transition kernel on $\mathbb{R}^d$ such that for $Y_x\sim P(x,\cdot)$, $x\in\mathbb R^d$, there exists a random variable $W_x$ with $\mathbb{E}\Vert W_x\Vert \leq K$ for a constant $K<\infty$ independent of $x$ and
	\begin{align}
	\label{eq:p-property}
	\|Y_x\| \leq \|x\| + \|W_x\|
	\qquad
	\text{almost surely.} 
	\end{align}
	Additionally, assume that there exists a radius $R > 0$, constants $\ell\in(0,1]$ and $\widetilde \ell \in [0,1)$ such that for all $x \in B_R(0)^c := \{x \in \mathbb R^d\colon \|x\| >R\}$ there is a set $ D_x \in \mc B(\mathbb{R}^d)$ satisfying
	\begin{enumerate}[(a)]
		\item \label{en: 1st}
		$P(x, D_x) \geq \ell$,
		\item \label{en: 2nd}
		$\sup_{y\in D_x} \|y\| \leq \widetilde \ell \|x\|$.
	\end{enumerate}
	Then $V(x) := \|x\|$ is a Lyapunov function for $P$ with $L:=R + K <\infty$ and $\delta := 1-(1-\widetilde{\ell})\ell<1$.
\end{prop}
\begin{proof}
	We distinguish whether $x\in B_R(0)$ or $x\in B_R(0)^c$. Consider the case $x\in B_R(0)$: By assumption we have a.s.
	$V(Y_x) \leq V(x)+ V(W_x)$, such that
	\begin{align*}
	PV(x) = \mathbb{E}V(Y_x)
	& 
	\leq
	V(x) + \mathbb{E}\Vert W_x \Vert 
		\ \ \leq R + K 
	 \\
	& \leq \delta V(x)+R + K. 
	\end{align*}
	Consider the case $x\in B_R(0)^c$:
	We have
	\begin{align*}
	PV(x)
	& 
	=
	\int_{D_x} V(y) \ P(x, \d y)	
	+
	\int_{D^c_x} V(y) \ P(x, \d y).
	\end{align*}
	For the first term we obtain
	\[
	\int_{D_x} V(y) \ P(x, \d y)	
	\underset{\eqref{en: 2nd}}{\leq}
	\widetilde \ell \ V(x)
	\ P(x,D_x).
	\]
	To bound the second term 
	observe that
	\begin{align*}
	\int_{D^c_x} V(y) \ P(x, \d y)
	& = \mathbb{E}(\mathbf{1}_{D^c_x}(Y_x)\, V(Y_x))\\
	& \underset{\eqref{eq:p-property}}{\leq} V(x) \mathbb{P}(Y_x\in D^c_x) + \mathbb{E}\Vert W_x\Vert.
	\end{align*}
	We have
	\[
	\P(Y_x \in D_x^c)
	=
	1 - \P(Y_x \in D_x)
	=
	1 - P(x, D_x)
	\] 
	and combining both estimates above yields
	\begin{align*}
	PV(x)
	& \leq
	\widetilde \ell \ V(x) \ P(x,D_x) + (1 - P(x,D_x))\, V(x) + L\\
	& = [1 - P(x,D_x) + \widetilde \ell \, P(x,D_x)]\, V(x) + L.
	\end{align*}
	By the fact that $1-(1-\widetilde{\ell})P(x,D_x) \leq \delta$ the assertion is proven.
\end{proof}
We apply this proposition in the context of elliptical slice sampling and obtain the following result.
\begin{lem}\label{lem:D_x}
Assume that there exists an $\alpha\in (0,1/\sqrt{2}]$ and $R>0$, such that $B_{\alpha\Vert x\Vert}(0) \subseteq G_{\rho(x)}$ for all $x\in B_R(0)^c$. 
Then, the function $V(x) := \|x\|$ is a Lyapunov function for $E$ 
with some $\delta\in[0,1)$ and  $L\in [0,\infty)$.
\end{lem}
\begin{proof}
	From \eqref{eq:p-def} we have for all $x, w \in \mathbb{R}^d$ and any $\theta \in [0, 2\pi]$ that
	\begin{align} \label{equ:p_norm}
	\|p(x,w,\theta) \|
	\leq
	\|x\| + \|w\|.
	\end{align}
	Thus, condition \eqref{eq:p-property} is satisfied for the transition kernel $E$ with $W_x \sim \mu_0$ being the random variable $W$ in line 1 of Algorithm~\ref{alg:ESS}.
	Next, we show that for any $x \in B_R(0)^c$ and $D_x:= B_{\alpha \|x\|}(0)$ the assumptions \eqref{en: 1st} and \eqref{en: 2nd} of Proposition~\ref{prop:Lyapunov} are satisfied for an $\ell \in (0,1]$ and an $\widetilde{\ell}\in [0,1)$.
	Obviously, $\sup_{y\in D_x} \|y\| \leq \widetilde \ell \|x\|$ for $\widetilde \ell = \frac{1}{\sqrt{2}} < 1$ even for all $x\in\mathbb{R}^d$. Thus, it is sufficient to find a number $\ell \in(0,1]$
	such that
	\[
	E(x,D_x)
	\geq
	\ell,
	\qquad
	\forall x\in B_R(0)^c.
	\]
	For this notice that the probability to move to a set $A \in \mathcal{B}(\mathbb{R}^d)$ after all trials described in the lines 6--13 of Algorithm~\ref{alg:ESS} is larger than the probability to move to $A$ after exactly one iteration of the loop.
	Thus, for any $x, w \in \mathbb{R}^d$, $t \in [0, \rho(x)]$ and $A \in \mathcal{B}(\mathbb{R}^d)$ we have 	\begin{align}
	\label{eq:sub_kernel_property}
	E_{x, w, t}(A)
	\geq
	\frac{1}{2 \pi}
	\int_0^{2\pi}
	\mathds{1}_{A \cap G_t}(p(x, w, \theta))
	\,\d \theta,
	\end{align}
	with $E_{x,w,t}$ as given in \eqref{eq:sub_kernel}.
	Further, notice that for any $x\in B_{R}(0)^c$ and any $t\in[0,\rho(x)]$ we have
	\[
	D_x = B_{\alpha \|x\|}(0)\subseteq G_{\rho(x)} \subseteq G_t.
	\]
	Defining $\widetilde{\Theta}$ to be a $[0, 2\pi]$-uniformly distributed random variable and using \eqref{eq:sub_kernel_property} we have for any $x \in B_{R}(0)^c, w \in \mathbb{R}^d$ and $t \in [0, \rho(x)]$ that
	\begin{align*}
	E_{x, w, t}(D_x)
	& \geq
	\frac{1}{2 \pi}
	\int_0^{2\pi}
	\mathds{1}_{D_x}(p(x, w, \theta))
	\d \theta \\
	& =
	\mathbb{P}
	\left(
	p(x, w, \widetilde{\Theta})
	\in
	D_x
	\right).
	\end{align*}
	Additionally, let $W \sim \mu_0$ be independent of $\widetilde{\Theta}$. Then we have for all $x \in B_{R}(0)^c$
	\begin{align}
	\label{eq:one_step}
	E(x, D_x)
	\geq
	\mathbb{P}
	\left(
	p(x, W, \widetilde{\Theta})
	\in
	D_x
	\right).
	\end{align}
	Hence, we need to study the event $\big\{p(x,W,\widetilde{\Theta}) \in D_x\big\}$ in more detail. 
	We have
	\[
	p(x,W,\widetilde{\Theta})\in D_x
	\; \Longleftrightarrow \;
	\|p(x,W,\widetilde{\Theta})\|
	\leq
	\alpha \|x\|,
	\]
	which is equivalent to
	\begin{align*}
	\|p(x,W,\widetilde{\Theta})\|^2
	& = 
	\|x\|^2 \cos^2(\widetilde{\Theta})
	+
	\|W\|^2 \sin^2(\widetilde{\Theta}) \\
	& \quad+
	2\langle x,W\rangle \sin(\widetilde{\Theta})\cos(\widetilde{\Theta}) \\
	& \leq
	\alpha^2 \|x\|^2,
	\end{align*}
	where $\langle \cdot, \cdot \rangle$ denotes the standard inner product on $\mathbb{R}^d$.
	Defining 
	\begin{align*}
	A_W &:= \|x\|^2 - \|W\|^2, \\
	B_W &:= 2\langle x,W \rangle, \\
	C_W &:= (2\alpha^2 - 1)\|x\|^2 -\|W\|^2,
	\end{align*}
	and using the trigonometric identities
	\begin{align*}
	\cos(2\theta)
	&=
	2\cos^2(\theta) - 1
	=
	1 - 2\sin^2(\theta), \\
	\sin(2\theta)
	&=
	2\cos(\theta)\sin(\theta),
	\end{align*}
	we have that $p(x,W,\widetilde{\Theta})\in D_x$ is equivalent to
	\begin{align*}
	A_W
	\cos(2\widetilde{\Theta})
	+
	B_W
	\sin(2\widetilde{\Theta})
	\leq
	C_W.
	\end{align*}
	Letting $\varphi_W\in[0,2\pi)$ be an angle satisfying
	\[
	\cos(\varphi_W)
	=
	\frac{A_W}{\sqrt{A_W^2+B_W^2}},
	\quad
	\sin(\varphi_W)
	=
	\frac{B_W}{\sqrt{A_W^2+B_W^2}},
	\] 
	and using the cosine of sum identity we get
	\begin{align*}
	\cos(2\widetilde{\Theta}-\varphi_W)
	&\leq
	\frac{C_W}{\sqrt{A_W^2+B_W^2}}.
	\end{align*}
	At this point we have
	\begin{multline}\label{eq:p-in-Ax}
	\left\{
	p(x,W,\widetilde{\Theta})\in D_x
	\right\}
	= \\
	\left\{
	\cos(2\widetilde{\Theta}-\varphi_W)
	\leq
	\frac{C_W}{\sqrt{A_W^2+B_W^2}}
	\right\}.
	\end{multline}	
	Note that $A_W,B_W,C_W,\varphi_W$ are all random variables which depend on $W$, but are independent of $\widetilde{\Theta}$.
	We aim to condition on the event $\|W\|^2\leq \frac{\alpha^2 R^2}{2 - \alpha^2}$. 
	In this case $C_W < 0$ and $A_W > 0$, such that
	\[
	0 > \frac
	{C_W}
	{\sqrt{A_W^2 + B_W^2}}
	\geq
	\frac{C_W}{A_W}
	=
	\frac
	{(2\alpha^2 - 1)\|x\|^2 - \|W\|^2}
	{\|x\|^2 - \|W\|^2}.
	\]
	The last fraction can be rewritten as
	\[
	\frac
	{
		(\alpha^2 - 1)(\|x\|^2 - \|W\|^2)
		+ \alpha^2 \|x\|^2 + (\alpha^2 - 2) \|W\|^2
	}
	{\|x\|^2 - \|W\|^2}
	\]
	or equivalently as
	\[
	(\alpha^2 - 1)
	+
	\frac{\alpha^2}{\|x\|^2 - \|W\|^2}
	\left(
	\|x\|^2
	+
	\frac{\alpha^2 - 2}{\alpha^2} \|W\|^2
	\right).
	\]
	The second term is non-negative, therefore, we have
	\begin{align}\label{eq:ABC_W-bound}
	\frac
	{C_W}
	{\sqrt{A_W^2 + B_W^2}}
	\geq
	\alpha^2 - 1
	>
	-1.
	\end{align}
	With 
	$
	\ell_R
	:=
	\mathbb{P}
	\left(
	\|W\|^2\leq \frac{\alpha^2 R^2}{2 - \alpha^2}
	\right)
	>0
	$
	we have
	\begin{multline*}
	\mathbb{P}
	\left(
	p(x,W,\widetilde{\Theta})\in D_x
	\right)
	\geq \\
	\mathbb{P}
	\left(
	p(x,W,\widetilde{\Theta})\in D_x
	\middle|
	\|W\|^2\leq \frac{\alpha^2 R^2}{2 - \alpha^2}
	\right)
	\ell_R.
	\end{multline*}
	Now using \eqref{eq:p-in-Ax} and \eqref{eq:ABC_W-bound} we have that
	\begin{multline*}
	\mathbb{P}
	\left(
	p(x,W,\widetilde{\Theta})\in D_x
	\right)
	\geq \\
	\mathbb{P}
	\left(
	\cos(2\widetilde{\Theta}-\varphi_W)
	\leq
	\alpha^2 - 1
	\middle|
	\|W\|^2\leq \frac{\alpha^2 R^2}{2 - \alpha^2}
	\right) \ell_R.
	\end{multline*}
	For any random variable $\xi$ independent of $\widetilde{\Theta}$ we have that the distribution of $\cos(2\widetilde{\Theta}-\xi)$ coincides with the distribution of $	\cos(2\widetilde{\Theta})$,
	since $\widetilde{\Theta}$ is uniformly distributed on $[0,2\pi]$. 
	Recall that $\varphi_W$ is independent of $\widetilde{\Theta}$.
	%
	%
	Therefore, 
	with
	\[
	\varepsilon_\alpha :=	\mathbb{P}
	\left(
	\cos(2\widetilde{\Theta})
	\leq
	\alpha^2 - 1
	\right) >0
	\]
	we have
	\[
	\mathbb{P}
	\left(
	\cos(2\widetilde{\Theta}-\varphi_W)
	\leq
	\alpha^2 - 1
	\middle|
	\|W\|^2\leq \frac{\alpha^2 R^2}{2 - \alpha^2}
	\right)
	=
	\varepsilon_\alpha.
	\]
	Putting everything together, we conclude that
	\[
	E(x,D_x)
	\overset{\eqref{eq:one_step}}{\geq}
	\mathbb{P}
	\left(
	p(x,W,\widetilde{\Theta})\in D_x
	\right)
	\geq
	\varepsilon_\alpha \ell_R > 0
	\]
	and all assumptions of Proposition~\ref{prop:Lyapunov} are then satisfied with $\ell:= \varepsilon_\alpha \ell_R$.
\end{proof}

\subsection{Small Set}
In this section we show that under suitable assumptions any compact set is small w.r.t. the transition kernel $E$ of elliptical slice sampling.

\begin{lem}
	\label{lem:small-set}
	Assume that $\rho$ is bounded away from $0$ and $\infty$ on any compact set. Then any compact set $G \subset \mathbb{R}^d$ is small w.r.t. $E$ and the measure $\varepsilon\cdot \lambda_{G}$, where $\varepsilon>0$ is some constant and $\lambda_{G}$ denotes the $d$-dimensional Lebesgue measure restricted to $G$.
\end{lem}
\begin{proof}
	Let $x\in G$ be arbitrary and recall that for any $A\in \mathcal{B}(\mathbb{R}^d)$ we have
	\[
	E(x,A) = \frac{1}{\rho(x)}\int_0^{\rho(x)} \int_{\mathbb{R}^d} E_{x,w,t}(A) \mu_0({\rm d}w){\rm d}t,
	\]
	where we argued in \eqref{eq:sub_kernel_property} that
	\[
	E_{x, w, t}(A)
	\geq
	\frac{1}{2 \pi}
	\int_0^{2\pi}
	\mathds{1}_{A \cap G_t}(p(x, w, \theta))
	\,\d \theta,
	\]
	for any $w\in\mathbb{R}^d$ and $t\in [0,\rho(x)]$. Therefore, we obtain
	\begin{multline*}
	E(x, A)
	\geq \\
	\frac{1}{2\pi\rho(x)}
	\int\limits_0^{\rho(x)}
	\int\limits_{\mathbb{R}^d}
	\int\limits_0^{2\pi}
	\mathds{1}_{A \cap G_t}(p(x, w, \theta))
	\d \theta
	\mu_0(\d w)
	\d t.
	\end{multline*}
	Changing the order of integration yields
	\begin{multline*}
	E(x, A)
	\geq\\
	\frac{1}{2\pi\rho(x)}
	\int_0^{\rho(x)}
	\int_0^{2\pi}
	\mathbb{E}
	\big(
	\mathds{1}_{A \cap G_t}(p(x, W, \theta))
	\big)
	\d \theta
	\d t
	\end{multline*}
	for some random vector $W \sim  \mathcal{N}(0, C)$. Define the auxiliary random vector
	$Z_{x,\theta} := p(x, W, \theta)$ with corresponding distribution $\nu_{x,\theta} := \mathcal{N}(x \cos(\theta), \sin^2 (\theta) C)$. Then
	\begin{align*}
	E(x, A)
	&\geq
	\frac{1}{2\pi\rho(x)}
	\int_0^{\rho(x)}
	\int_0^{2\pi}
	\mathbb{E}
	\big(
	\mathds{1}_{A \cap G_t}(Z_{x,\theta})
	\big)
	\d \theta
	\d t
	\\
	&=
	\frac{1}{2\pi\rho(x)}
	\int_0^{\rho(x)}
	\int_0^{2\pi}
	\int_A
	\mathds{1}_{G_t}(z)
	\nu_{x,\theta}(\d z)
	\d \theta
	\d t.
	\end{align*}
	Using the fact that $\mathds{1}_{G_t}(z) = \mathds{1}_{[0, \rho(z)]}(t)$ we have
	\begin{align*}
	& E(x, A)
	\geq \\ 
	& \qquad\int_0^{2\pi}
	\int_A
	\frac{1}{2\pi\rho(x)}
	\int_0^{\rho(x)}
	\mathds{1}_{[0, \rho(z)]}(t)
	\d t \,
	\nu_{x,\theta}(\d z)
	\d \theta.
	\end{align*}
	Notice that
	\begin{align*}
	\frac{1}{\rho(x)}
	\int_0^{\rho(x)}
	\mathds{1}_{[0, \rho(z)]}(t)
	\d t
	& =
	\frac{1}{\rho(x)} \min \{\rho(x), \rho(z)\} \\
	& =
	\min \left\{1, \frac{\rho(z)}{\rho(x)} \right\}.
	\end{align*}
	Moreover, for all $x, z \in G$ by the boundedness assumption on $\rho$ we have
	\[
	\min \left\{1, \frac{\rho(z)}{\rho(x)} \right\}
	\geq 
	\min \left\{1, \frac{\inf_{a \in G} \rho(a)}{\sup_{a \in G}\rho(a)} \right\}
	=:
	\beta > 0.
	\]
	Thus,
	\begin{align*}
	E(x, A)
	&\geq
	\frac{\beta}{2\pi} 
	\int_0^{2\pi}
	\nu_{x,\theta}(A \cap G)
	\d \theta
	\\
	&\geq
	\frac{\beta}{2\pi} 
	\int_{\frac{\pi}{4}}^{\frac{\pi}{2}}
	\nu_{x,\theta}(A \cap G)
	\d \theta.
	\end{align*}
	Since $G$ is a compact set, there exists a finite constant $\kappa>0$, such that
	\[
	(z - x \cos \theta)^T
	C^{-1}
	(z - x \cos \theta)
	\leq
	\kappa,
	\qquad
	\forall x, z \in G.
	\]
	Moreover, for all $\theta \in \left[\frac{\pi}{4}, \frac{\pi}{2}\right]$ we have that $\frac{1}{2} \leq \sin^2 (\theta) \leq 1$.
	Therefore, the factors of the density of the Gaussian distribution $\nu_{x,\theta}$ satisfy
	\[
	\exp\left(-
	\frac
	{ 
		(z - x \cos \theta)^T
		C^{-1}
		(z - x \cos \theta)
	}
	{2 \sin^2(\theta)}\right)
	\geq
	\exp(-\kappa),\]
	and
	\[
	\left(
	2 \pi \sin^2 (\theta)
	\right)
	^{-\frac{d}{2}}
	\det(C) ^ {-\frac{1}{2}}
	\geq
	(2 \pi)^{-\frac{d}{2}}
	\det(C) ^ {-\frac{1}{2}}.
	\]
	Hence,
	\[	
	\nu_{x,\theta}(A \cap G)
	\geq 	
	\frac{\exp(- \kappa)}{(2 \pi)^{\frac{d}{2}}
		\det(C) ^ {\frac{1}{2}}}
	\lambda_{G} (A),
	\]
	such that finally with
	$\varepsilon := \frac{\beta}{8} \frac{\exp(-\kappa)}{(2 \pi)^{\frac{d}{2}}
		\det(C) ^ {\frac{1}{2}}} $ we have
	\[
	E(x, A)
	\geq
	\varepsilon\cdot \lambda_{G} (A),
	\]
	which finishes the proof.
\end{proof}
\begin{rem}
	For a compact set $G\subset \mathbb{R}^d$ suppose that $\rho\colon G\to (0,\infty)$
	with
	$
	0< \inf_{x\in G} \rho(x) $
	and
	$
	\sup_{x\in G} \rho(x) <\infty.
	$
	In this setting the same arguments as in the proof of Lemma~\ref{lem:small-set} can be used to verify that the whole state space $G$ is small w.r.t. elliptical slice sampling. This leads to the fact that elliptical slice sampling is uniformly ergodic in this scenario, see for example Theorem~15.3.1 in \cite{douc2018markov}. For a summary of different ergodicity properties and their relations to each other we refer to Section~3.1 in \cite{rudolf2012explicit}.
\end{rem}

\subsection{Proof of Theorem~\ref{thm:main-result}}
We apply Proposition~\ref{prop:Harris_erg_thm}. First, recall that in \cite{murray2010elliptical} it is verified that elliptical slice sampling is reversible w.r.t. $\mu$ and therefore $\mu$ is a stationary distribution of $E$. Hence, it is sufficient to provide a Lyapunov function and to check the smallness of $S_R$. By Assumption~\ref{assum:reg} part \ref{assum:good-tails}. the requirements for Lemma~\ref{lem:D_x} are satisfied, such that $V(x):=\Vert x \Vert$ is a Lyapunov function with $\delta\in [0,1)$ and $L\in[0,\infty)$. By Assumption~\ref{assum:reg} part \ref{assum:pi-bounded-away}. using Lemma~\ref{lem:small-set} we obtain that for any $R>2L/(1-\delta)$ the set
$S_R = B_R(0)$ is compact and therefore small w.r.t. transition kernel $E$ and some non-trivial finite measure. Therefore, all requirements of Proposition~\ref{prop:Harris_erg_thm} are satisfied and the statement of  Theorem~\ref{thm:main-result} follows.

\section{Illustrative Examples}
\label{sec:Illus}

In this section we verify in toy scenarios as well as more demanding settings the conditions of Assumption~\ref{assum:reg} to illustrate the applicability of our result. 
In Section~\ref{sec:suppl-exp-family} of the supplementary we provide a discussion in terms of the exponential family.

\subsection{Gaussian Posterior}
\label{sec:Gauss}

In \cite{murray2010elliptical} Gaussian regression is considered as test scenario for elliptical slice sampling, since there the posterior distribution is again Gaussian. We see covering that setting as a minimal requirement for our theory: Here, for some $x_0\in\mathbb{R}^d$ we have
\begin{equation}
\label{eq: rho_Gaussian}
\rho(x) = \exp\left(-\frac12 (x-x_0)^T \Sigma^{-1} (x-x_0)\right),
\quad x\in \mathbb{R}^d,
\end{equation}
that is, $\rho$ is proportional to a Gaussian density with non-degenerate covariance matrix $\Sigma$. Thus, the matrix $\Sigma\in\mathbb{R}^{d\times d}$ is symmetric, positive-definite, and we denote its eigenvalues by $\lambda_1,\dots,\lambda_d$. Notice that all eigenvalues are strictly positive and
define $\lambda_{\min} := \min_{i=1,\dots,d} \lambda_i$, $\lambda_{\max} := \max_{i=1,\dots,d} \lambda_i$. The covariance matrix induces a norm $\Vert \cdot \Vert_{\Sigma^{-1}}$ on $\mathbb{R}^d$ by
\[
\Vert x \Vert_{\Sigma^{-1}}^2 = x^T \Sigma^{-1} x. 
\]
It is well-known that the Euclidean and the $\Sigma^{-1}$-norm are equivalent. One has
\begin{equation} 
\label{eq: sigma_Eucl_equi}
\lambda_{\max}^{-1} \Vert x\Vert^2
\leq \Vert x \Vert_{\Sigma^{-1}}^2 \leq \lambda_{\min}^{-1} \Vert x\Vert^2,\quad \forall x\in\mathbb{R}^d.	
\end{equation}
Now we are able to formulate and prove the following proposition guaranteeing the applicability of Theorem~\ref{thm:main-result}.
\begin{prop}\label{propo:Gaussian_rho}
	For $\rho$ defined in \eqref{eq: rho_Gaussian} Assumption~\ref{assum:reg} is satisfied with $R=4\sqrt{\frac{\lambda_{\max}}{\lambda_{\min}}} \Vert x_0 \Vert$ and
	$\alpha  = \frac{1}{2} \sqrt{\frac{\lambda_{\min}}{\lambda_{\max}}}$.
\end{prop}
\begin{proof}
	Observe that $\rho$ is continuous, bounded by $1$ and strictly larger than $0$ everywhere, such that part \ref{assum:pi-bounded-away}. of Assumption~\ref{assum:reg} is true. By exploiting both inequalities in \eqref{eq: sigma_Eucl_equi} we show part \ref{assum:good-tails}. of Assumption~\ref{assum:reg}, that is, we verify for all
	$ x\in B_{4\sqrt{\frac{\lambda_{\max}}{\lambda_{\min}}} \Vert x_0 \Vert}(0)^c$ holds
	$
	B_{\frac{1}{2} \sqrt{\frac{\lambda_{\min}}{\lambda_{\max}}}\Vert x\Vert}(0) \subseteq G_{\rho(x)}.
	$
	For this fix $x\in B_{4\sqrt{\frac{\lambda_{\max}}{\lambda_{\min}}} \Vert x_0 \Vert}(0)^c$. Therefore, we have 
	\begin{equation} \label{eq:rel_x_to_x0}
	\Vert x \Vert \geq 4 \sqrt{\frac{\lambda_{\max}}{\lambda_{\min}}}\, \Vert x_0 \Vert.
	\end{equation}
	Now let $y\in 	B_{\frac{1}{2} \sqrt{\frac{\lambda_{\min}}{\lambda_{\max}}}\Vert x\Vert}(0)$. Therefore, we have
	\begin{equation} \label{eq:rel_y_to_x}
	\Vert y \Vert \leq 	\frac{1}{2} \sqrt{\frac{\lambda_{\min}}{\lambda_{\max}}}\Vert x\Vert,
	\end{equation}
	and one might observe that
	\[
	G_{\rho(x)} 
	= \{ y\in\mathbb{R}^d\colon \Vert y-x_0 \Vert_{\Sigma^{-1}} \leq \Vert x-x_0 \Vert_{\Sigma^{-1}} \}.
	\]
	With this we obtain
	\begin{align*}
	&	\Vert y-x_0 \Vert_{\Sigma^{-1}}
	\leq \Vert y \Vert_{\Sigma^{-1}} + \Vert x_0 \Vert_{\Sigma^{-1}}   \underset{\eqref{eq: sigma_Eucl_equi}}{\leq}\frac{ \Vert y \Vert}{\lambda_{\min}^{1/2}} + \Vert x_0 \Vert_{\Sigma^{-1}}\\ 
	& \underset{\eqref{eq:rel_y_to_x}}{\leq}
	\frac{\Vert x \Vert}{2 \lambda_{\max}^{1/2}}  + \Vert x_0 \Vert_{\Sigma^{-1}} 
	\underset{\eqref{eq: sigma_Eucl_equi}}{\leq}
	\frac{\Vert x \Vert_{\Sigma^{-1}}}{2} + \Vert x_0 \Vert_{\Sigma^{-1}}\\
	& = \Vert x \Vert_{\Sigma^{-1}} - \Vert x_0 \Vert_{\Sigma^{-1}} - \frac{\Vert x \Vert_{\Sigma^{-1}}}{2} + 2 \Vert x_0 \Vert_{\Sigma^{-1}} \\
	& \underset{\eqref{eq: sigma_Eucl_equi}}{\leq} \Vert x-x_0\Vert_{\Sigma^{-1}} - \frac{\Vert x \Vert}{2 \lambda_{\max}^{1/2}} + 2 \Vert x_0 \Vert_{\Sigma^{-1}}\\
	& \underset{\eqref{eq:rel_x_to_x0}}{\leq}
	\Vert x-x_0\Vert_{\Sigma^{-1}} - \frac{2\Vert x_0\Vert}{\lambda_{\min}^{1/2}} + 2 \Vert x_0 \Vert_{\Sigma^{-1}}
	\underset{\eqref{eq: sigma_Eucl_equi}}{\leq}
	\Vert x-x_0\Vert_{\Sigma^{-1}}
	\end{align*}
	which provides the desired result.
\end{proof}
In Gaussian process regression as well as Bayesian inverse problems with linear forward maps the resulting posterior distribution has again a Gaussian density $\rho$ with respect to the Gaussian prior $\mu_0$.
However, in these applications the corresponding covariance matrix of $\rho$ is typically positive semi-definite, and we have to replace $\Sigma^{-1}$ in \eqref{eq: rho_Gaussian} by its pseudo-inverse $\Sigma^{\dagger}$.
We emphasize that also in this more general situation Assumption~\ref{assum:reg} is satisfied, since $\rho$ is then simply constant on the null space of $\Sigma$ and on its orthogonal complement we can apply Proposition \ref{propo:Gaussian_rho}.

\subsection{Multi-modality}
In the previous section we considered the setting of a Gaussian posterior distribution $\mu$. In particular, $\rho$ had just a single peak. It seems that such a requirement is not necessary to verify the crucial Assumption~\ref{assum:reg}. Here we introduce a class of density functions which might behave almost arbitrarily in their ``center'' (the central part of the state space) and exhibit a certain tail behavior. For formulating the result, let $\vert\cdot \vert$ be a norm on $\mathbb{R}^d$ which is equivalent to the Euclidean norm $\Vert \cdot\Vert$, that is, there exist constants $c_1,c_2\in (0,\infty)$ such that
\begin{align}
\label{eq:norm-equiv-Eucl}
c_1 \Vert x \Vert \leq \vert x\vert \leq c_2\Vert x\Vert,\quad \forall x\in\mathbb{R}^d.
\end{align}
\begin{prop}
	\label{prop:multi}
	For some $R'>0$ and some $x_0 \in \mathbb{R}^d$ let $\rho_{R'}\colon B_{R'}(x_0) \to (0,\infty)$ be continuous and let $r\colon [R',\infty) \to (0,\infty)$ be decreasing.
	Furthermore, suppose that
	\begin{equation}  \label{eq:inf_sup}
	\inf_{z\in B_{R'}(x_0)} \rho_{R'}(z) \geq \sup_{t\geq R'} r(t).
	\end{equation}
	Then, the function
	\[
	\rho(x) := \begin{cases}
	\rho_{R'}(x) & x\in B_{R'}(x_0) \\
	r(\vert x -x_0\vert) & x\in B_{R'}(x_0)^c,
	\end{cases}
	\]
	satisfies Assumption~\ref{assum:reg} with $R = \max\{R', 4 \frac{c_2}{c_1}\|x_0\| \}$ and $\alpha= \frac{c_1}{2c_2}$.	
\end{prop} 
\begin{proof}
	By the continuity of $\rho_{R'}$ and the fact that $r$ is strictly positive as well as decreasing  part \ref{assum:pi-bounded-away}. of Assumption~\ref{assum:reg} is satisfied. For part \ref{assum:good-tails}. let $x\in B_R(0)^c$, i.e., $\|x\| > R'$ and $\|x\| > 4\frac{c_2}{c_1}\|x_0\|$.
	Hence, we have by \eqref{eq:inf_sup} and the decreasing property of $r$ that
	\[
	G_{\rho(x)} = B_{R'}(x_0) 
	\cup \{ y\in B_{R'}(x_0)^c\colon
	\vert y -x_0\vert \leq \vert x -x_0\vert \}.
	\] 
	Now let $y\in B_{\alpha \|x\|}(0)$ and distinguish two cases: 
	\begin{enumerate}
		\item For $y\in B_{R'}(x_0)$ we immediately have $y\in G_{\rho(x)}$, and we are done.
		\item For $y\in B_{R'}(x_0)^c \cap B_{\alpha \|x\|}(0)$ we obtain due to $\|y \| \leq \frac{c_1}{2c_2} \|x\|$ that
		\[
			|y - x_0|
			\leq
			|y| + |x_0|
			\underset{\eqref{eq:norm-equiv-Eucl}}{\leq}
			\frac12 |x| + |x_0|
		\]
		and, furthermore, by exploiting $\|x\| > 4\frac{c_2}{c_1}\|x_0\|$ that
		\begin{align*}
			\frac12 |x| + |x_0|
			&
				= |x| - |x_0| - \frac12 |x| + 2 |x_0| \\
			&
				\underset{\eqref{eq:norm-equiv-Eucl}}{\leq}
					|x - x_0| - 2 |x_0| + 2 |x_0|
			= |x - x_0|,
		\end{align*}
		which leads again to $y\in G_{\rho(x)}$.
	\end{enumerate}
	Both cases combined yield the statement. 
\end{proof}
To state an example which satisfies the assumption of Proposition~\ref{prop:multi} we consider the following ``volcano density''.
\begin{example}
	Set $\rho(x) := \exp(\Vert x\Vert - \frac{1}{2}\Vert x\Vert^2)$. 
	Let $\vert \cdot \vert = \Vert \cdot \Vert$, 
	$x_0=0$, $R' = 2$, 
		$r(t):= \exp(t-t^2/2)$ and $\rho_{R'}$ be the restriction of $\rho$ to $B_2(0)$. It is easily checked that for this choice of parameters all required properties are satisfied. One can argue that the function $\rho$ is highly multi-modal, since its maximum is attained on a $d-1$-dimensional manifold (a sphere). For illustration, it is plotted in Figure~\ref{fig:volcano_denisty}.
\end{example}
\begin{figure}
	\begin{center}
		\includegraphics[width=0.4\textwidth]{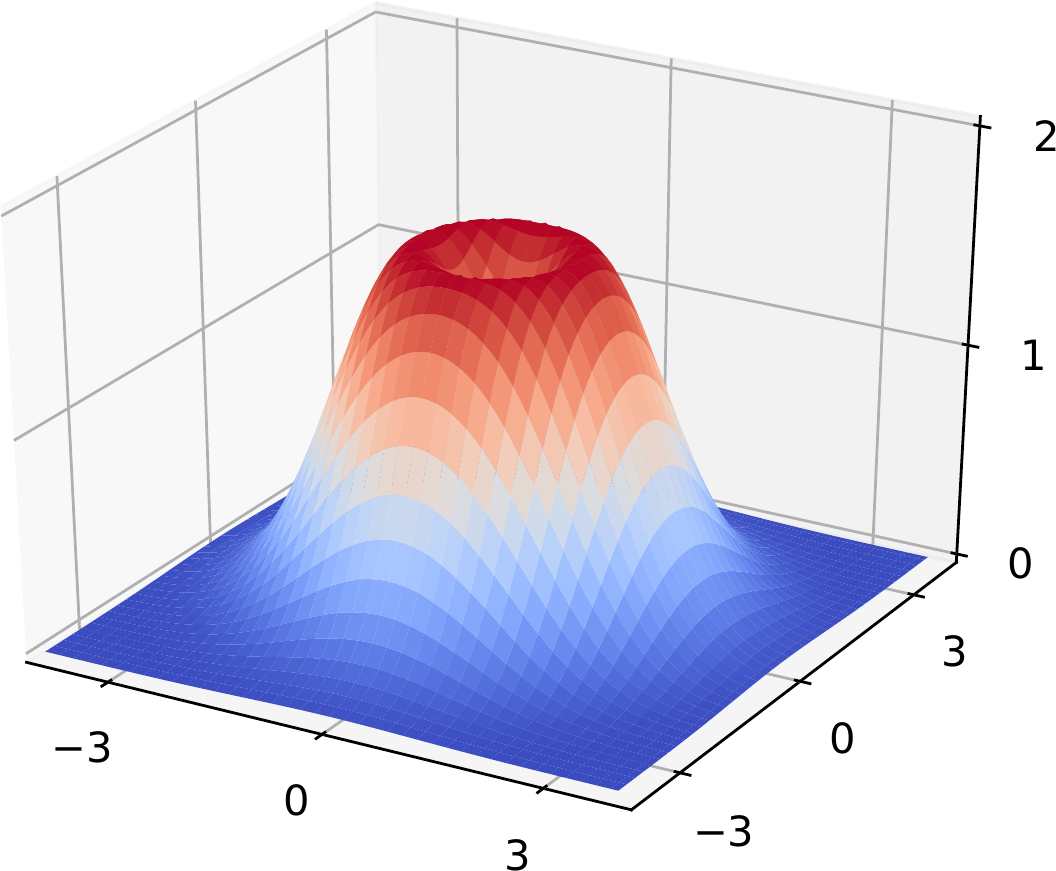}
	\end{center}
	\caption{Plot of the function $x\mapsto \exp(\Vert x \Vert - \frac{1}{2}\Vert x \Vert^2)$ for $d=2$.}
	\label{fig:volcano_denisty}
\end{figure}

\subsection{Volcano Density and Limitations of the Result}
\label{sec:volcano-density}
In the last section we showed the applicability of Theorem~\ref{thm:main-result} for a ``volcano density''. 
Here we use 
this density 
differently. 
Namely, 
$
\mu(\d x) \propto \exp\left( \Vert x \Vert - \frac{1}{2} \Vert x\Vert^2\right) \d x,
$
that is, the Lebesgue density of $\mu$ is proportional to the function plotted in Figure~\ref{fig:volcano_denisty}.
Setting $\mu_0 = \mathcal{N}(0,I)$ with identity matrix $I$, we obtain
\begin{align}
\label{eq:rho-exploding}
\rho(x) = \exp(\Vert x\Vert), \quad x\in\mathbb{R}^d.
\end{align}
Observe that in this setting for any $x\in \mathbb{R}^d$ we have
\[
G_{\rho(x)} = \{ y\in \mathbb{R}^d \colon \Vert y \Vert \geq \Vert x \Vert\} = B_{\Vert x\Vert}(0)^c,
\]
such that $G_{\rho(x)}$ never completely contains a ball around the origin and Assumption~\ref{assum:reg} cannot be satisfied. 
For this scenario we conduct numerical experiments in various dimensions, namely,  $d=10,30,100,300,1000$.
Although, our sufficient Assumption~\ref{assum:reg} is not satisfied%
\footnote{
		In Section~\ref{sec:suppl-tail-shift-modification} of the supplementary material we provide further discussions how Assumption~\ref{assum:reg} can be  satisfied in this scenario by taking a modification into account.
}, we still observe a good performance of the elliptical slice sampler. 
In particular, its statistical efficiency in terms of the \emph{effective sample size (ESS)} seems to be independent of the dimension, see Figure~\ref{fig:ESS}.
To check whether this ``dimension-independent'' behavior is inherently due to the particular setting or not, we also consider other Markov chain based sampling algorithms. 

For estimating the ESS we use an empirical proxy
of the autocorrelation function 
\[
	\gamma_f(k)
	:=
	\mathrm{Corr}(f(X_{n_0}), f(X_{n_0+k})),
\]
of the underlying Markov chain $(X_k)_{k\in\mathbb{N}}$ for a chosen quantity of interest $f\colon\mathbb{R}^d \to \mathbb{R}$ where $n_0$ denotes a burn-in parameter. 
Since the ESS takes the form
\[
	\mathrm{ESS}(n, f, (X_k)_{k\in\mathbb{N}})
	=
	n
	\left(
	1 + 2 \sum_{k=0}^\infty \gamma_f(k)
	\right)^{-1},
\]
where $n\in\mathbb{N}$ denotes the chosen sample size, we approximate it by
using the empirical proxy of $\gamma_f(k)$ and truncating the summation at $k = 10^4$.

In Figure \ref{fig:ESS} we display estimates of the ESS for four different Markov chain Monte Carlo algorithms. 
\begin{figure}
		\begin{center}
	\includegraphics[width=0.48\textwidth]{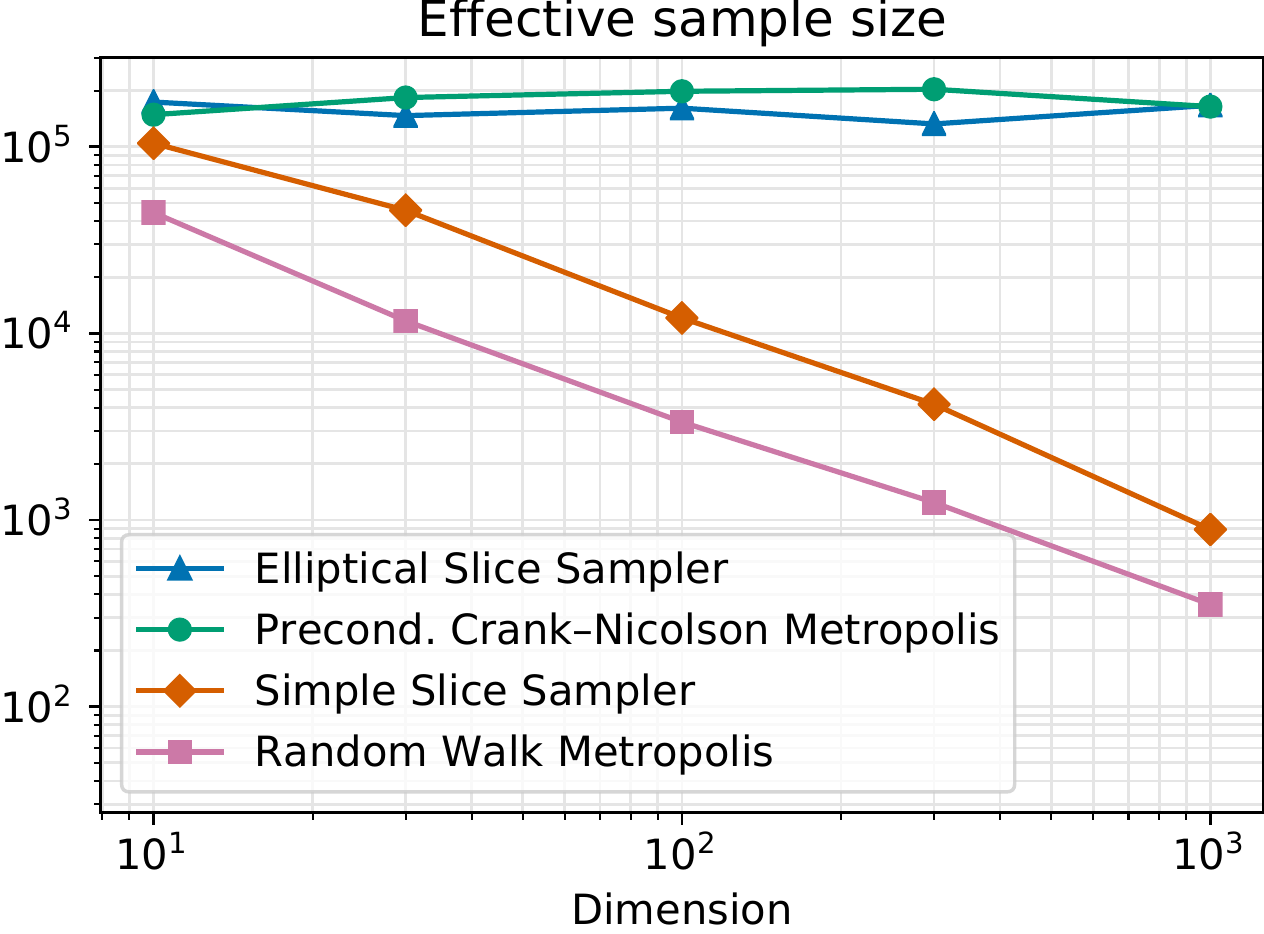}
		\end{center}
	\caption{Proxies for ESS for different MCMC algorithms depending on the	dimension of the space.}
	\label{fig:ESS}
\end{figure}
Namely, the random walk Metropolis algorithm (RWM), the preconditioned Crank-Nicolson Metropolis (pCN), the simple slice sampler and the elliptical one.
For each algorithm we set the initial state to be $0\in \mathbb{R}^d$ and compute the ESS for $f(x) := \log(1+\Vert x\Vert)$, $n_0 := 10^5$ and $n:= 10^6$.
Both Metropolis algorithms (the RWM and the pCN Metropolis) were tuned to an averaged acceptance probability of approximately $0.25$. 
We clearly see in Figure \ref{fig:ESS} the dimension-dependence of the ESS for the simple slice sampler\footnote{
	In the light of \cite{natarovskii2021quantitative} 
		the dimension-dependent behavior for simple slice sampling is not surprising. There, for a certain class of $\rho$ a
		spectral gap of size $1/d$ is proven.} and the RWM.
In contrast to that, the results for the elliptical slice sampler and the pCN Metropolis indicate a dimension-independent efficiency. 
Let us remark 
that elliptical slice sampling does not need to be tuned in comparison to the pCN Metropolis, which performs similarly.
However, the price for this 
is the requirement of evaluating the function $\rho$ more often within a single transition.
Here the function $\rho$ was evaluated on average $1.5$ times in each iteration of the elliptical slice sampler. 
Intuitively, the example of this section is not covered by our theorem, since the tail behavior of $\rho$ is ``bad''. Namely, for $\|x\| \to \infty$ we have $\rho(x) \to \infty$. It seems that for convergence only the tail behavior of likelihood times prior considered as Lebesgue density matters.

Let us briefly comment on different approaches how to verify the numerically observed dimension-independence. Similarly to the strategy employed in \cite{hairer2014spectral,rudolf2018generalization} 
	for the pCN Metropolis 
	one might be able to extend elliptical slice sampling on infinite-dimensional Hilbert spaces. If one proves the existence of an absolute spectral gap of the correspondent transition kernel, then this directly gives bounds of the total variation distance of the $n$th step distribution to the stationary one. Due to the infinite-dimensional setting one might argue that the estimate must be independent of the dimension. Another approach is to prove dimension-free Wasserstein contraction rates, as, for example, has been done in \cite{eberle2016reflection,eberle2019quantitative,de2019convergence} for diffusion processes. 
\subsection{Logistic Regression}
\label{sec: log_regr}
Suppose data $(\xi_i, y_i)_{i=1,\dots,N}$ with $\xi_i \in \mathbb{R}^d$ and $y_i \in \{-1, 1\}$ for $i=1,\dots,N$ is given.  
For logistic regression the function
$\rho\colon\mathbb{R}^d\to(0,\infty)$ 
takes the form
\begin{align}
\rho(x)
=
\prod_{i=1}^N
\frac
{1}
{1 + \exp(-y_i x^T \xi_i)},
\quad
x \in \mathbb{R}^d.
\end{align}
Moreover, assume we have a Gaussian prior distribution $\mu_0$ on $\mathbb{R}^d$ with $\mu_0=\mathcal{N}(0,I)$. Thus, the distribution of interest, i.e., the posterior 
distribution $\mu$ is determined by
$
\mu(\d x) \propto \rho(x) \,\mu_0(\d x).
$
The function $\rho$ does not satisfy Assumption~\ref{assum:reg}, since it has no vanishing tails. For example for $d = N = \xi_1 = y_1 = 1$ we have  $\rho(x) = (1 + \exp(-x))^{-1}$, which is increasing with 
$
G_{\rho(x)}
=
[x, \infty)$ for all $
\forall x \in \mathbb{R}$.
Thus, $\rho$ cannot satisfy Assumption~\ref{assum:reg}. 
In the general setting the phenomena is the same and the arguments are similar. 

Therefore, our theory for elliptical slice sampling seems not to be applicable. However, with a ``tail-shift'' modification we can satisfy Assumption~\ref{assum:reg}. The idea is to take a ``small'' part of the Gaussian prior
and shift it to the likelihood function, such that it gets sufficiently nice tail behavior.  

For arbitrary $\varepsilon\in(0,1)$ set $\widetilde{\mu}_0:=\mathcal{N}(0,(1-\varepsilon)^{-1}I)$ and
\begin{equation}
\label{eq: tail-shift-2}
\widetilde{\rho}(x) := \rho(x) \exp\left(-\varepsilon\Vert x\Vert^2/2\right).
\end{equation}
Observe that $\widetilde{\rho}$ has, in contrast to $\rho$, exponential tails.
Moreover, note that $\mu_0(\d x) \propto \exp(-\varepsilon\Vert x\Vert^2/2) \widetilde{\mu}_0(\d x)$ and therefore
\[
\mu(\d x) \propto \rho(x)\mu_0(\d x)
\propto \widetilde{\rho}(x) \widetilde{\mu}_0(\d x).
\]
Now considering $\mu$ as given through $\widetilde\rho$ and $
\widetilde{\mu}_0$ our main theorem is applicable.
In Section~\ref{sec:suppl-logistic-regression} of the supplementary material we prove the following result and provide a discussion of the ``tail-shift'' modification.
\begin{prop}  \label{prop:log_regr}
	For $\varepsilon\in(0,1)$ the function $\widetilde{\rho}$ given in \eqref{eq: tail-shift-2} 
	satisfies Assumption~\ref{assum:reg} for $\alpha=\varepsilon/2$ and $R=4 N \min_{i=1,\dots,N} \Vert \xi_i \Vert / \varepsilon$.
\end{prop}
Finally, note that for having the guarantee of geometric ergodicity of elliptical slice sampling one can choose $\varepsilon\in(0,1)$ arbitrarily small, whereas for $\varepsilon=0$ our theory does not apply.

\section{Conclusion}
\label{sec:Concl}
In this paper we provide a mild sufficient condition for the geometric ergodicity of the elliptical slice sampler in finite dimensions.
In particular, it is satisfied if the density of the target measure with respect to a Gaussian measure $\mu_0$ is continuous, strictly positive and has a sufficiently nice tail behavior.
Besides that our numerical results indicate that (a) our condition is not necessary and (b) the elliptical slice sampler shows a dimension-independent efficiency. 
Both issues will be addressed in future research.

\section*{Acknowledgements}
We thank the anonymous referees for their valuable remarks, in particular, for bringing the ``tail-shift'' modification to our attention.
VN thanks the DFG Research Training Group 2088 for their support.  BS 
	acknowledges support of the DFG within project 389483880. DR gratefully acknowledges support of the DFG within project 432680300 -- SFB 1456 (subproject B02).


\bibliographystyle{icml_style/icml2021}
\bibliography{references}


\appendix

\twocolumn[
\icmltitle{
Supplementary Material to\\
``Geometric Convergence of Elliptical Slice Sampling''
}

\vskip 0.3in
]




\section{Derivation of Proposition~\ref{prop:Harris_erg_thm}}
\label{sec:suppl-thm}
We comment on deriving Proposition~\ref{prop:Harris_erg_thm} (formulated in the article) from the results in \cite{hairer2011yet}.
For stating the Harris ergodic theorem shown in \cite{hairer2011yet} we 
need 
to introduce the following weighted supremum norm. 
For a chosen weight function $V\colon\mathbb{R}^d \to [0,\infty)$ and for $\varphi\colon \mathbb{R}^d \to \mathbb{R}$ define 
\[
\Vert \varphi \Vert_V := \sup_{x\in\mathbb{R}^d} \frac{\vert\varphi(x)\vert}{1+V(x)}.
\]
One may think of $V$ as the Lyapunov function of a generic transition kernel $P$. 
Now we state Theorem~1.2 from \cite{hairer2011yet} on $\mathbb{R}^d$.

\begin{thm} \label{thm:harris-thm}
	Let $P$ be a transition kernel on $\mathbb{R}^d$. Assume that $V\colon\mathbb{R}^d \to [0,\infty)$ is a Lyapunov function of $P$ with $\delta\in[0,1)$ and $L\in[0,1)$.
	Additionally, for some constant $R>2L/(1-\delta)$ let
	\[
		S_R := \{x\in\mathbb{R}^d \colon V(x)\leq R\}
	\]
	be a small set w.r.t. $P$ and a non-zero measure $\nu$ on $\mathbb{R}^d$. Then, there is a unique stationary distribution $\mu_\star$ of $P$ on $\mathbb{R}^d$ and there exist constants $\gamma\in (0,1)$ as well as $C<\infty$ such that
	\begin{align}
	\label{eq:harris-thm}
	\|P^n \varphi - \mu_{\star}(\varphi)\|_V
	\leq
	C \gamma^n \|\varphi - \mu_{\star}(\varphi)\|_V,
	\end{align}
	where $P^n \varphi(x) := \int_{\mathbb{R}^d} \varphi(y) P^n(x,\d y)$ and $\mu_\star(\varphi) := \int_{\mathbb{R}^d} \varphi(y) \mu_\star(\d y)$ for any $x\in\mathbb{R}^d$ as well as any $n\in\mathbb{N}$.
\end{thm}
Let us assume that all requirements of the previous theorem are satisfied. Then,
for any $x\in \mathbb{R}^d$ we have
\begin{align*}
 & \Vert P^n(x,\cdot) - \mu_\star\Vert_{\rm tv} 
  = \sup_{\Vert f \Vert_\infty \leq 1} \left \vert P^n f(x) - \mu_\star(f) \right \vert \\
 & \leq \sup_{\Vert \varphi \Vert_{V} \leq 1} 
 \left \vert P^n \varphi(x) - \mu_\star(\varphi) \right \vert \\
 & =(1+V(x)) \sup_{\Vert \varphi \Vert_{V} \leq 1} \frac{
 \left \vert P^n \varphi(x) - \mu_\star(\varphi) \right \vert}{1+V(x)}\\
 & \leq (1+V(x)) \sup_{\Vert \varphi \Vert_{V} \leq 1} \Vert P^n \varphi - \mu_\star(\varphi) \Vert_V\\
 & \leq (1+V(x)) C \gamma^n \sup_{\Vert \varphi \Vert_{V} \leq 1} \Vert \varphi - \mu_\star(\varphi) \Vert_V \\
 & \leq 2(1+V(x)) C \gamma^n,
\end{align*}
which shows that the statement of Proposition~\ref{prop:Harris_erg_thm} is a consequence of Theorem~\ref{thm:harris-thm}.

	\section{Further Example from the Exponential Family}
	\label{sec:suppl-exp-family}

	We formulate a consequence of Proposition~\ref{prop:multi} (stated in the article) in terms of properties of the exponential family and provide examples which eventually satisfy our regularity condition. 
	For the convenience of the reader we repeat the assumption which guarantees the applicability of the main theorem.
	\begin{assumption}
		\label{assum:reg_appendix}
		The function $\rho\colon\mathbb{R}^d\to (0,\infty)$ satisfies the following properties:
		\begin{enumerate}
			\item \label{assum:suppl-pi-bounded-away}
			It is bounded away from $0$ and $\infty$ on any compact set.
			\item	\label{assum:suppl-good-tails}
			There exists an $\alpha>0$ and $R>0$, such that
			\[
			B_{\alpha\|x\|}(0) \subseteq G_{\rho(x)}  \qquad \text{for } \|x\| >R. 
			\]
		\end{enumerate}
	\end{assumption}
It is clear that regularity properties for members of the exponential family are required, since already by part \ref{assum:suppl-pi-bounded-away}. of the former assumption we need that 
$\rho$ has full support. For example, $\rho$ coming from the exponential distribution does not work, since then it is not bounded away from $0$ on any compact set where $\rho$ is equal to $0$.

   Let $|\cdot|$ be a norm on $\mathbb{R}^d$, which is equivalent to the Euclidean norm $\|\cdot\|$, that is, there exist constants $c_1,c_2\in (0,\infty)$ such that
   \[
   c_1 \|x\| \leq |x| \leq c_2 \|x\|,
   \qquad
   \forall x\in\mathbb{R}^d.
   \]
   We obtain the following result:
\begin{cor}  \label{cor:exp-family}
	Let $\rho$ be proportional to the mapping
	\[
	x\mapsto\exp(\eta(x)^T \mu - A(x)),
	\qquad
	x \in \mathbb{R}^d,
	\]
	for some $\eta: \mathbb{R}^d \to \mathbb{R}^k$, $\mu \in \mathbb{R}^k$ and $A: \mathbb{R}^d \to \mathbb{R}$ with $k \in \mathbb{N}$.
Assume that there exists an
	increasing
	function $\varphi: [0, \infty) \to \mathbb{R}$ as well as a point $x_0 \in \mathbb{R}^d$, such that
	\[
	\eta(x)^T \mu - A(x)
	=
	-\varphi(|x - x_0|),
	\qquad
	\forall x \in \mathbb{R}^d,
	\]
	or equivalently, such that $\rho$ is proportional to the mapping
	\[
	x\mapsto 
	\exp(-\varphi(|x - x_0|)),
	\qquad
	x \in \mathbb{R}^d.
	\]
	Then $\rho$ satisfies Assumption~\ref{assum:reg_appendix} with $R = 4\frac{c_2}{c_1} \|x_0\|$ and $\alpha = \frac{c_1}{2 c_2}$.
\end{cor}
\begin{proof}
	Apply Proposition~\ref{prop:multi} from the article with arbitrary $R'>0$, function $r(t):=\exp(-\varphi(t))$ and $\rho_{R'}(x)=\exp(-\varphi(|x - x_0|))$ defined on $B_{R'}(x_0)$.
\end{proof}
	Now we illustrate how to use the former corollary.
	\subsection{Gaussian density}
    Despite having the Gaussian setting already covered in Section~\ref{sec:Gauss} of the article, we show that this canonical member of the exponential family can also be treated with Corollary~\ref{cor:exp-family}.

	For any $x_0 \in \mathbb{R}^d$ and any symmetric, positive-definite matrix $\Sigma\in\mathbb{R}^{d\times d}$ the classical Gaussian setting, where
	\[
		\rho(x)
		=
		\exp
		\left(
			-\frac12 (x-x_0)^T \Sigma^{-1} (x-x_0)
		\right),
		\quad
		x\in \mathbb{R}^d,
	\]
	corresponds to a member of the exponential family with $k = 1$, $\mu = -1$, $A(x)=0$ and
	\[
		\eta(x)
		=
		\frac12
		(x-x_0)^T \Sigma^{-1} (x-x_0).
	\]
	It can be rewritten as
	\[
		\rho(x)
		=
		\exp(-\varphi(\|x - x_0\|_{\Sigma^{-1}})),
		\quad
		x\in \mathbb{R}^d,
	\]
	with the continuous increasing function $\varphi(t) = t$ and a norm $\vert \cdot \vert = \|\cdot\|_{\Sigma^{-1}}$, defined by
	\begin{align}
		\label{eq:Sigma-norm}
		\|x\|_{\Sigma^{-1}}
		:=
		x^T \Sigma^{-1} x.
	\end{align}
	Note that the norm is equivalent to the Euclidean one since
	\[
		\lambda_{\max}^{-1} \|x\|^2
		\leq
		\|x\|_{\Sigma^{-1}}^2
		\leq
		\lambda_{\min}^{-1} \|x\|^2,
		\quad
		\forall x\in\mathbb{R}^d,	
	\]
	where $\lambda_{\min}$ is the smallest and $\lambda_{\max}$ is the largest eigenvalue of the symmetric, positive-definite matrix $\Sigma$.
	Thus, all requirements of Corollary~\ref{cor:exp-family} are satisfied and therefore Assumption~\ref{assum:reg_appendix} is fulfilled.
		
	\subsection{Multivariate $t$-distribution}
	For any $\nu>1, x_0 \in \mathbb{R}^d$ and any symmetric, positive-definite matrix $\Sigma$ we have
	\[ \rho(x) =
		\left(
			1
			+
			\frac{1}{\nu}
			(x-x_0)^T \Sigma^{-1} (x-x_0)
		\right)
		^{-(\nu+d)/2},
	\]
	for $x \in \mathbb{R}^d$. This 
	corresponds to a member of the exponential family with $k = 1$, $\mu = -1$, $A(x)=0$ and
	\[
		\eta(x) = \frac{\nu+d}{2}
		\log \left(
			1
			+
			\frac{1}{\nu}
			(x-x_0)^T \Sigma^{-1} (x-x_0)
		\right).
	\]
	Using $\vert \cdot \vert = \|\cdot\|_{\Sigma^{-1}}$
	as defined in \eqref{eq:Sigma-norm}
	and the fact that $\varphi\colon [0,\infty) \to \mathbb{R}$, given by \[
	\varphi(t)
	:=
	\frac{\nu+d}{2}
	\log \left(1 + \frac{1}{\nu} t\right),
	\quad
	t \geq 0,
	\]  
	is increasing we can apply Corollary~\ref{cor:exp-family} and therefore Assumption~\ref{assum:reg_appendix} is satisfied.

	\section{``Tail-Shift'' Modification}
	\label{sec:suppl-tail-shift-modification}
	\label{sec:modification}
	If $\rho\colon \mathbb{R}^d\to(0,\infty)$ has ``poor'' tail behavior and therefore does not satisfy Assumption~\ref{assum:reg_appendix}, as e.g. in the scenario of the ``volcano density'' or logistic regression considered in the article, then a ``tail-shift'' modification might help. The idea is to take a small part of the Gaussian prior and shift it to $\rho$ to get sufficiently ``nice'' tails.   

	Assume that the distribution of interest $\mu$ is determined by $\rho\colon \mathbb{R}^d\to(0,\infty)$ and prior distribution $\mu_0=\mathcal{N}(0, C)$, that is,
	\[
	\mu(\d x)
	\propto
	\rho(x) \mu_0(\d x).
	\]
	For arbitrary $\varepsilon \in (0, 1)$ 
	set 
	\[
		f(x)
		:=
		\exp
		\left(
			-\frac{\varepsilon}{2}
			x^T C^{-1} x
		\right),
		\qquad
		x \in \mathbb{R}^d,
	\]
	and $\widetilde{\mu}_0 := \mathcal{N}(0, (1-\varepsilon)^{-1}C))$. Note that
	\begin{align}
		\label{eq:mu_0-density}
		\mu_0(\d x)
		\propto
		f(x) \widetilde{\mu}_0(\d x).
	\end{align}
	The function $f$ represents the part of $\mu_0$ which we shift from the prior to $\rho$. For doing this rigorously we define 
	\begin{align}
		\label{eq:rho-tilde}
		\widetilde{\rho}(x)
		:=
		\rho(x) f(x),\quad x\in \mathbb{R}^d,
	\end{align}
	and obtain an alternative representation of $\mu$. Namely,
	\[
\mu(\d x)\propto
\rho(x) \mu_0(\d x) \underset{\eqref{eq:mu_0-density}}{\propto}
\rho(x) f(x) \widetilde{\mu}_0(\d x)
\underset{\eqref{eq:rho-tilde}}{=}
\widetilde{\rho}(x) \widetilde{\mu}_0(\d x).
\]
Using the representation of $\mu$ in terms of $\widetilde \rho$ and $\widetilde\mu_0$ it might be possible to satisfy Assumption~\ref{assum:reg_appendix} for $\widetilde{\rho}$ as the following example shows.

\begin{example}
	We apply the ``tail-shift'' modification to the
	``volcano density'' considered in Section~\ref{sec:volcano-density} in the article. Recall that
	\[
		\rho(x) = \exp(\Vert x\Vert), \quad x\in\mathbb{R}^d,
	\]
	and $\mu_0=\mathcal{N}(0,I)$. For $\varepsilon\in(0,1)$ after setting
	\[
	f(x):= \exp\left(-\frac{\varepsilon}{2}\Vert x\Vert^2\right),
	\]
	we obtain $\mu_0(\d x) \propto f(x)\widetilde{\mu}_0(\d x)$ with $\widetilde{\mu}_0 = \mathcal{N}(0,(1-\varepsilon)^{-1}I)$ and
	\[
		\widetilde{\rho}(x) = \exp\left(\Vert x\Vert - \frac{\varepsilon}{2} \Vert x \Vert^2 \right).
	\]
	By applying Proposition~\ref{prop:multi} from the article with $\vert \cdot \vert = \Vert \cdot \Vert$, $x_0=0$, $R'=2 \varepsilon^{-1}$ and $r(t):=\exp(t-\varepsilon t^2/2)$ as well as $\rho_{R'}$ being the restriction of $\widetilde{\rho}$ to $B_{R'}(0)$ we get that Assumption~\ref{assum:reg_appendix} is satisfied.
\end{example}
    We want to emphasize here that different representations of $\mu$ lead, eventually, to different algorithms. Observe that one can choose $\varepsilon\in(0,1)$ arbitrarily small and the requirements for the main theorem are satisfied, whereas for $\varepsilon=0$ our theory does not apply. Unfortunately it is not always easy to verify Assumption~\ref{assum:reg_appendix} in the modified setting.
    
	In the following, we provide another tool for showing Assumption~\ref{assum:reg_appendix}. 
	Independent of the ``tail-shift'' modification it can be used to prove that for certain $\rho\colon\mathbb{R}^d \to (0,\infty)$ the main theorem is applicable.
	\begin{prop}
		\label{prop:rho-bound}
		For $\rho\colon \mathbb{R}^d\to (0,\infty)$ and some $R>0$ suppose that there are continuous functions $\rho_\ell: \mathbb{R}^d \to (0, \infty)$ and $\rho_u: \mathbb{R}^d \to (0, \infty)$, such that
		\begin{align}
		\label{eq:rho-tilde-bound}
		\rho_\ell(x) \leq \rho(x) \leq \rho_u(x),
		\qquad
		\forall x \in \mathbb{R}^d.
		\end{align}
		Furthermore, assume that for some $\alpha>0$ we have
		\begin{align}
		\label{eq:level-sets-rho-tilde-rho-l}
		A_x
		:=
		\{y \in \mathbb{R}^d: \rho_\ell(y) \geq \rho_u(x)\}
		\supseteq
		B_{\alpha \|x\|}(0)
		\end{align}
		for any $x \in B_R(0)^c$.
		Then $\rho$ satisfies Assumption~\ref{assum:reg_appendix} with constants $R$ and $\alpha$.
	\end{prop}
	\begin{proof}
		Obviously, $\rho$ is bounded away from $0$ and $\infty$ on any compact set, since $\rho_\ell$ and $\rho_u$ are strictly positive and continuous.
		Therefore, part~\ref{assum:suppl-pi-bounded-away}. of Assumption~\ref{assum:reg_appendix} is satisfied.
		For part~\ref{assum:suppl-good-tails}. notice that for all $x \in B_R^c(0)$ holds
		$
			G_{\rho(x)}
			\supseteq
			A_x,
		$
		since, if $y \in A_x$, then $\rho_\ell(y) \geq \rho_u(x)$ and therefore
		\[
			\rho(y)
			\underset{\eqref{eq:rho-tilde-bound}}{\geq}
			\rho_\ell(y)
			\geq
			\rho_u(x)
			\underset{\eqref{eq:rho-tilde-bound}}{\geq}
			\rho(x).
		\]
		Thus,
		\[
			G_{\rho(x)}
			\supseteq
			A_x
			\supseteq
			B_{\alpha \|x\|}(0),
			\qquad
			\forall x \in B_R(0)^c,
		\]
		which finishes the proof.
	\end{proof}
We apply the former proposition to the logistic regression example and therefore prove Proposition~\ref{prop:log_regr} from the article.
	\subsection{Logistic Regression}
	\label{sec:suppl-logistic-regression}
	For some data $(\xi_i, y_i)_{i=1,\dots,N}$ with $\xi_i \in \mathbb{R}^d$ and $y_i \in \{-1, 1\}$ for $i=1,\dots,N$ let
	\begin{align}
		\label{eq:rho-log-regr}
		\rho(x)
		=
		\prod_{i=1}^N
		\frac
			{1}
			{1 + \exp(-y_i x^T \xi_i)},
		\quad
		x \in \mathbb{R}^d.
	\end{align}
	In this case $\rho$ does not satisfy Assumption~\ref{assum:reg_appendix}, see Section~\ref{sec: log_regr} in the main article. Using the ``tail-shift'' modification changes the picture. 

Let $\mu_0 = \mathcal{N}(0, I)$ and note that for arbitrary $\varepsilon \in (0, 1)$, with
\[
f(x) := \exp(-\varepsilon \|x\|^2/2),
\qquad
\theta \in \mathbb{R}^d,
\]
 the measure $\mu_0$ can be expressed as
\[
	\mu_0(\d x) \propto f(x) \widetilde{\mu}_0(\d x)
\]
with $\widetilde{\mu}_0 := \mathcal{N}(0, (1 - \varepsilon)^{-1}I)$.
	Therefore, $\widetilde{\rho}$ from \eqref{eq:rho-tilde} takes the form
	\[
		\widetilde{\rho}(x)
		=
		\exp(-\varepsilon \|x\|^2/2)
		\prod_{i=1}^N
		\frac
			{1}
			{1 + \exp(-y_i x^T \xi_i)}.
	\]
	Observe that $\widetilde{\rho}$ has, in contrast to $\rho$, exponential tails.
	To apply Proposition~\ref{prop:rho-bound} to $\widetilde{\rho}$ we need to find
	suitable lower and upper bounds which satisfy the
	 conditions formulated in \eqref{eq:rho-tilde-bound} and \eqref{eq:level-sets-rho-tilde-rho-l}.
	For any $x \in \mathbb{R}^d$ we have
	by applying the Cauchy-Schwarz inequality that
	\[
		\exp(-\beta \|x\|)
		\leq
		\rho(x)	
		\leq
		1,
	\]
	where $\beta := 2 N \underset{i=1,\dots,N}{\min} \|\xi_i\|$.
	Taking this into account, 
	with
	\begin{align*}
		\rho_\ell(x)
		&:=
		\exp(-\varepsilon \|x\|^2/2)
		\exp(-\beta \|x\|), \\
		\rho_u(x)
		&:=
		\exp(-\varepsilon \|x\|^2/2),
	\end{align*}
	we have the desired lower and upper bound for $\widetilde{\rho}$.
	For $A_x$ defined in \eqref{eq:level-sets-rho-tilde-rho-l} (based on $\rho_\ell$ and $\rho_u$) we show that
	\begin{align}
		\label{eq:A_theta}
		A_{x}
		\supseteq
		\left\{
			z \in \mathbb{R}^d
			\colon
			\|z\|
			\leq
			\frac{\varepsilon}{2} \|x\|
		\right\}
	\end{align}
	for all $x \in \mathbb{R}^d$ with $\|x\| \geq 2\beta / \varepsilon$.
	For this notice that 
	\begin{align*}
		A_{x}
		& =
			\left\{
				z \in \mathbb{R}^d
				\colon
				-\beta \|z\| - \varepsilon \|z\|^2 / 2
				\geq
				-\varepsilon \|x\|^2 / 2
			\right\} \\
		& =
			\left\{
				z \in \mathbb{R}^d
				:
				\varepsilon \|z\|^2 + 2 \beta \|z\| - \varepsilon \|x\|^2
				\leq
				0
			\right\} \\
		& =
			\left\{
				z \in \mathbb{R}^d
				:
				\|z\|
				\leq
				-\beta + \sqrt{\beta^2 + \varepsilon^2 \|x\|^2}
			\right\} \\	
		& \supseteq
			\left\{
				z \in \mathbb{R}^d
				:
				\|z\|
				\leq
				\varepsilon \|x\| - \beta
			\right\},
	\end{align*}
	where the inclusion is due to the fact that $\sqrt{\beta^2 + \varepsilon^2 \|x\|^2} \geq \varepsilon \|x\|$.
	We conclude that for any $x \in \mathbb{R}^d$ with $\|x\| \geq 2\beta / \varepsilon$, or equivalently, $\beta \leq \varepsilon \|x\| / 2$, condition~\eqref{eq:A_theta} holds true.
	Thus, all requirements of Proposition~\ref{prop:rho-bound} are fulfilled for $\alpha = \varepsilon/2$ and $R = 2\beta / \varepsilon$ and therefore $\widetilde{\rho}$ satisfies Assumption~\ref{assum:reg_appendix}.
	
	We summarize that the application of the main theorem, which gives geometric ergodicity of elliptical slice sampling, depends on the representation of $\mu$. As pointed out
	for
	\[
		\mu(\d x) \propto \rho(x) \mu_0(\d x),
	\]
	with $\rho\colon \mathbb{R}^d \to (0,\infty)$ and $\mu_0=\mathcal{N}(0,C)$, it might be possible that Assumption~\ref{assum:reg_appendix}	is not satisfied. Therefore, for elliptical slice sampling with this representation of $\mu$ we do not provide any ergodicity guarantee. However, 
	by using the ``tail-shift'' modification it is likely that one can find $\widetilde{\rho}\colon\mathbb{R}^d \to (0,\infty)$ and a Gaussian measure $\widetilde{\mu}_0$ with 
		\[
	\mu(\d x) \propto \widetilde\rho(x) \widetilde{\mu}_0(\d x),
	\]
	such that for $\widetilde{\rho}$ Assumption~\ref{assum:reg_appendix} is satisfied and the geometric ergodicity theorem for elliptical slice sampling is applicable for $\widetilde{\rho}$ and $\widetilde{\mu}_0$.

\end{document}